\documentclass[letterpaper, 10 pt, conference]{ieeeconf}
\usepackage[english]{babel}
\IEEEoverridecommandlockouts                              

\usepackage[table]{xcolor}
\usepackage{amsmath}
\usepackage{amsfonts}
\usepackage{graphicx}
\usepackage{verbatim}
\usepackage{todonotes}
\usepackage[colorlinks=true, allcolors=blue]{hyperref}
\usepackage[linesnumbered,ruled,noend]{algorithm2e}
\usepackage{bm}
\usepackage[font=scriptsize]{caption}
\usepackage{subcaption}

\newcommand{\X}{\mathcal{X}}
\newcommand{\Xsafe}{\mathcal{X}_\textrm{safe}}
\newcommand{\U}{\mathcal{U}}

\newcommand{\PDM}{\mathcal{S}_{>0}}
\newcommand{\nx}{n_x}
\newcommand{\dimu}{n_u}
\newcommand{\ufb}{u_\text{fb}}
\newcommand{\trackerr}{\eta}
\newcommand{\T}{\mathcal{T}}

\newtheorem{theorem}{Theorem}

\title{\LARGE \bf Statistical Safety and Robustness Guarantees for Feedback Motion \\ Planning of Unknown Underactuated Stochastic Systems}

\author{Craig Knuth$^{1}$, Glen Chou$^{2}$, Jamie Reese$^{1}$, Joseph Moore$^{1,3}$
\thanks{$^{1}$Johns Hopkins University Applied Physics Laboratory, Laurel, MD, USA.
        {\tt\small Craig.Knuth@jhuapl.edu, Jamie.Reese@jhuapl.edu, Joseph.Moore@jhuapl.edu}}%
\thanks{$^2$University of Michigan, Ann Arbor, MI, USA.
    {\tt\small gchou@umich.edu}}%
\thanks{$^3$ Johns Hopkins University Department of Mechanical Engineering, Baltimore, MD, USA.}
}

\begin{document}
\maketitle

\begin{abstract}
We present a method for providing statistical guarantees on runtime safety and goal reachability for integrated planning and control of a class of systems with unknown nonlinear stochastic underactuated dynamics. Specifically, given a dynamics dataset, our method jointly learns a mean dynamics model, a spatially-varying disturbance bound that captures the effect of noise and model mismatch, and a feedback controller based on contraction theory that stabilizes the learned dynamics. We propose a sampling-based planner that uses the mean dynamics model and simultaneously bounds the closed-loop tracking error via a learned disturbance bound. We employ techniques from Extreme Value Theory (EVT) to estimate, to a specified level of confidence, several constants which characterize the learned components and govern the size of the tracking error bound. This ensures plans are guaranteed to be safely tracked at runtime. We validate that our guarantees translate to empirical safety in simulation on a 10D quadrotor, and in the real world on a physical CrazyFlie quadrotor and Clearpath Jackal robot, whereas baselines that ignore the model error and stochasticity are unsafe.
\end{abstract}

\section{Introduction}

To deploy robots in the real world, we need assurances on runtime safety and robustness. However, real environments and hardware present a myriad of challenges to attaining such guarantees, such as unknown dynamics, system stochasticity arising from non-repeatable environmental interactions, and underactuation. These challenges establish safe real-world planning and control as a difficult open problem. 
To handle unknown system dynamics, a model is often obtained via black-box machine learning, e.g., via neural networks (NNs). Such methods are flexible and can accurately model complex phenomena, like driving on rough terrain or moving through fluid media like water or air. However, model prediction error can vary across the relevant domain.
This can be catastrophic for planning and control, which can exploit inaccuracies to produce unreasonable plans that cannot be safely tracked.

To ensure that plans computed using the learned model can be safely tracked on the true system, we must characterize the noise profile of the true system together with the model mismatch, and identify how these errors can impact the performance of downstream planning and control. For unknown stochastic systems, this ultimately requires a statistical argument, as only a finite set of sample rollouts can be obtained on the true system and statements on safety must be made using this data. Therefore, we apply methods from Extreme Value Theory (EVT) \cite{de2007extreme}, a suite of statistical methods for estimating the maxima and minima of unknown distributions from data, to bound the model error in a trusted domain $D$ and propagate its effect through the planner and controller to make statistical guarantees on closed-loop behavior. Our method builds upon prior work that has applied EVT for safe planning with learned models of \textit{deterministic} fully actuated \cite{lipschitz_ral} and underactuated \cite{ChouOB21} systems with \textit{noiseless} training data. These methods use the \textit{Lipschitz constant} (i.e., the maximum rate of change) of the model error to inform planning. However, these approaches are unusable for stochastic systems, where the Lipschitz constant becomes unbounded due to noise (see Fig. \ref{fig:bound}). This corresponds to infinite model error bounds which are useless for informing downstream planning. This limitation precludes \cite{lipschitz_ral, ChouOB21} from scaling to real-world robots where noise and stochasticity are inevitable. 

\begin{figure}
    \centering
    \includegraphics[width=0.8\linewidth]{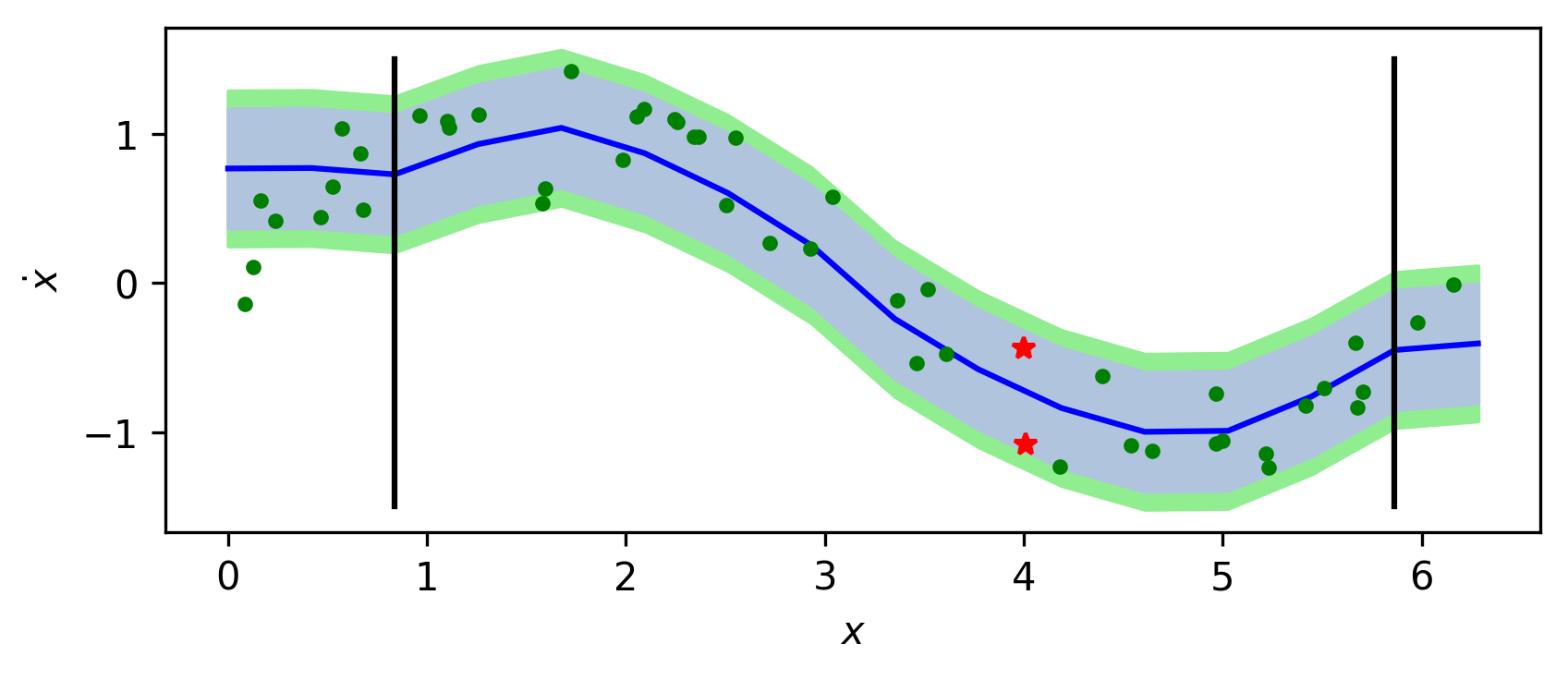}
    \caption{Comparison of our noise bounding approach compared to \cite{lipschitz_ral}, \cite{ChouOB21} for example system $h(x) = \dot{x}$. From data (green dots), our method estimates a bound (shaded blue) around a learned model (blue) in a domain $D$ (between black lines). This bound is verified with the addition of a small buffer (shaded green). Note this bound is not valid outside of $D$. Pairwise slopes on the data are used to estimate the Lipschitz constant in \cite{lipschitz_ral}, \cite{ChouOB21}, but the ratio of change in $\dot{x}$ to change in $x$ is unbounded due to noise, e.g. the red star data points.}
    \label{fig:bound}
    \vspace{-16pt}
\end{figure}

To address this gap, we propose a novel method to bound the error in the learned model which can be applied on stochastic systems and \textit{on real robots}. We bound how this model error can impact a downstream tracking controller based on contraction theory, which applies to a broad class of underactuated systems, 
providing statistically-guaranteed closed-loop tracking tubes for any plan in the trusted domain. Our planner is a modified RRT that considers the tracking tube in calculating collision checks, ensures feedback control can be faithfully executed, and ensures that plans stay within the trusted domain where the EVT-based analysis is valid. This ensures runtime safety and goal reachability up to the statistical confidence. This confidence does not degrade with time, as it is made with respect to the model and controller, rather than specific plans. We summarize our method in Fig. \ref{fig:flowchart}. Our contributions:

\begin{itemize}
    \item A novel model error bound that can be applied on stochastic systems with noisy training data
    \item Derivations for propagating this error bound through the planner and controller, and designing a planner that uses the resulting tracking bound to return safer plans
    \item Validation on a simulated 10D quadrotor, a real 6D CrazyFlie quadrotor, and a real 5D Clearpath Jackal, outperforming baselines in safety and robustness
\end{itemize}

\begin{figure*}[ht]
    \centering
    \includegraphics[width=\textwidth]{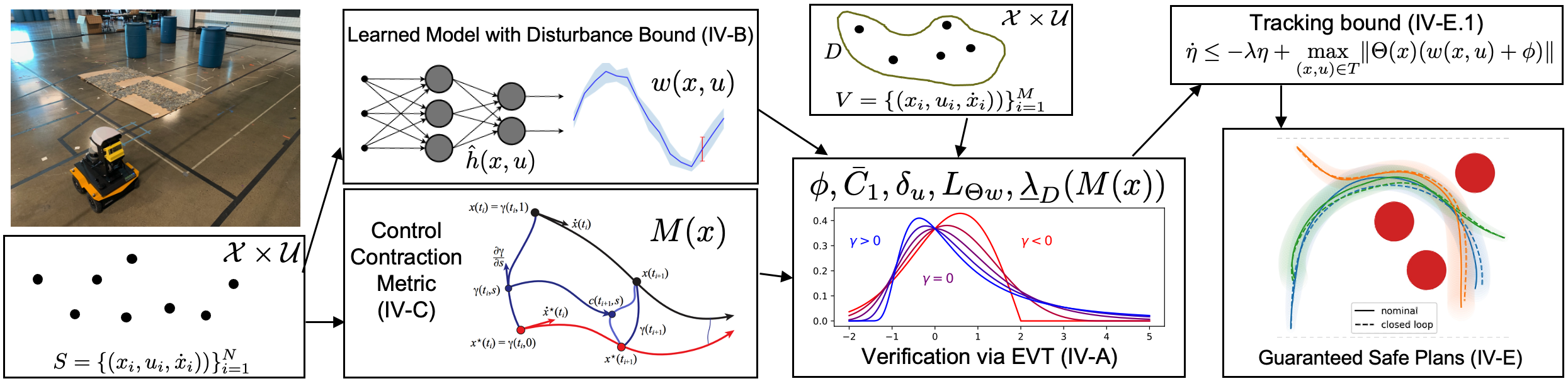}\vspace{-8pt}
    \caption{Flow chart of method. CCM figure sourced from \cite{manchester}. Sections given where applicable.\vspace{-15pt}}
    \label{fig:flowchart}
\end{figure*}

\section{Related Work}


Our work is related to feedback motion planning of uncertain systems. One body of work \cite{hj, MajumdarT17, sumeet_icra} ensures safe feedback motion planning under worst-case error via tools from reachability analysis and contraction theory, under known, constant model error bounds. In contrast, we use a tighter learned spatially-varying error bound, which we also use to inform planning. Adaptive contraction-based controllers are another solution, but existing methods \cite{DBLP:journals/corr/abs-2004-01142, DBLP:journals/corr/abs-2003-10028} assume a known model uncertainty structure, i.e., that it lies in the span of known basis functions. We consider uncertainty arising from the model error and noise distribution, which have no known structure. Stochastic contraction \cite{DBLP:journals/tac/PhamTS09, DBLP:journals/csysl/TsukamotoCS21}, i.e., contraction in expectation, has been studied, but does not easily provide high-probability tracking tubes. Learning-based method to contraction-based control are well-suited for controlling the deep dynamics models we consider; however, existing methods  \cite{dawei}, \cite{DBLP:journals/csysl/TsukamotoC21} assume known dynamics and constant error bounds, or do not consider tracking error \cite{sumeet_wafr}. 

Our work is also related to safe planning for stochastic systems via tools like chance constraints \cite{DBLP:conf/icra/HanJW22, DBLP:journals/ral/ChouWB22, prob_decomp, DBLP:conf/corl/ChouBO20, vitus, ccrrt, estimate_prob_collision_sensing_uncertainty} and occupation measures \cite{DBLP:conf/corl/MengSQWF21, DBLP:journals/ijrr/MajumdarVTT14}. These methods reason about an (approximate) distribution over states, but in doing so, require accurate knowledge of the noise distribution, and have difficulties scaling (in state dimension and plan length \cite{DBLP:conf/icra/HanJW22, DBLP:journals/ral/ChouWB22}). To sidestep these issues, we do not seek out detailed distributional information, instead solving a simpler robust planning problem that uses a high-probability estimate on the \textit{support} of the disturbance.

Finally, our work is related to safe learning-based control. One class of methods learns stability certificates: \cite{DBLP:conf/nips/KolterM19},\cite{DBLP:journals/corr/abs-2008-05952} learns Lyapunov functions from data; however, it is difficult to apply a single Lyapunov function for general point-to-point motion planning. \cite{DBLP:journals/corr/abs-2005-00611} learns a model-free stability certificate, but does not consider control design. Other methods use Gaussian processes (GPs) for reachable tube estimation \cite{koller2018learning} and safe exploration \cite{akametalu2014reachability, berkenkamp2016safe}, but assume that a controller is given; we do not make these assumptions. \cite{DBLP:journals/corr/abs-2002-01587} learns trajectory tracking tubes; however, plans must remain near full training \textit{trajectories} to be accurate, which can heavily restrict planning; we only require plans to be near \textit{state/controls} seen in training. Finally, most relevant are \cite{lipschitz_ral, ChouOB21}, which plan safely with learned dynamics by deriving an invariant tube around a plan inside a ``trusted domain" near the training data. A restrictive key assumption of \cite{lipschitz_ral, ChouOB21} is that the unknown system is \textit{deterministic}, preventing these approaches from scaling to real-world systems with stochasticity and noise. In this paper, we make these methods applicable to a broad class of stochastic systems by making fundamental improvements to \cite{lipschitz_ral, ChouOB21}, i.e., in designing new a model error bound and deriving its impact on tracking error. 

\section{Preliminaries and Problem Formulation}

Let $\dot{x} = h(x,u)$ be the true continuous stochastic dynamics where $h: \X \times \U \rightarrow \X$, $x \in \X$ is the state, $u \in \U$ is the control, and $\dot{x}$ is the time derivative of $x$, i.e., for each $x, u$, there is a distribution of possible state derivatives. Let $\texttt{dim}(\X) = \nx$ and $\texttt{dim}(\U) = \dimu$. Furthermore, let $\dot{x} = \hat{h}(x,u) = f(x) + g(x) u$ be a deterministic control-affine approximation of the dynamics that is paired with an estimate of the maximum disturbance at $(x,u)$ denoted $w(x,u): \X \times \U \rightarrow \mathbb{R}^+_{\nx}$ where $\mathbb{R}^+_{\nx}$ is the set of positive real vectors with dimension $\nx$. Let $S = \{(x_i,u_i,
\dot{x}_i)\}_{i=0}^N$ and $V = \{(x_i,u_i,\dot{x}_i)\}_{i=0}^M$ be a training and validation dataset of $N$ and $M$ transitions collected on the true system respectively. Let $D = D_x \times D_u \subset \X \times \U$ be the trusted domain partitioned into state and control subsets.

Let $\overline{Q}$ be the symmetric part of $Q$, i.e. $\overline{Q} = \tfrac{1}{2}(Q + Q^\top)$. We denote $\bar{\lambda}(Q)$ and $\underline{\lambda}(Q)$ as the largest and smallest eigenvalues of $Q$ respectively. Furthermore for matrix valued functions $Q(x)$ we define $\bar{\lambda}_D(Q(x)) = \max_{x \in D} \bar\lambda(Q(x))$, similarly for  $\underline{\lambda}_D(Q(x))$. $|v|$ refers to the element-wise absolute value of a vector $v$. $v_k$ denotes the $k$th element of $v$. We define addition between a vector $v$ and scalar $a$ as $a$ added to each element of $v$. Let $\bm{0}_{n \times m}$ denote a zero matrix 
and $\mathbf{I}_{n \times n}$ denote the identity matrix%
.

Let $\PDM$ be the set of positive definite matrices with dimension $\nx \times \nx$. For a smooth manifold $\X$, a Riemannian metric is defined as $M(x) : \X \rightarrow \PDM$. 
The distance between two points $x_1$ and $x_2$, $d(x_1,x_2)$, is defined as the length of the geodesic between them (see \cite{ChouOB21}).
In this paper $d$ refers to Riemannian distance, $\|\cdot\|$ refers to the Euclidean norm, and $\|\cdot\|_{\bar{M}}$ refers to the metric defined under $\bar{M} \in \PDM$. We also note that the Euclidean distance between two points $x_1$, $x_2$ on a Riemannian manifold with metric $M(x)$ can be bounded as follows: $\|x_1 - x_2\| \leq d(x_1,x_2) / \underline{\lambda}_D(M(x))$ (see \cite{ChouOB21}). Let a ball be defined as $\mathcal{B}_{\bar{M}}(\chi,r) \doteq \{\chi' \,|\, \|\chi' - \chi\|_{\bar{M}} \leq r\}$. If $\bar{M}$ is omitted, then the Euclidean norm is used.

Note the specific form of $D$ is left unspecified. We only require the ability to check if a ball of some radius $r_1$ fits in the domain, i.e. $\mathcal{B}_{\bar{M}}((x,u),r_1) \subset D$ (see \cite{lipschitz_ral}). In practice, we define $D$ either as a union of $\ell_2$-balls with radius $r_2 \geq r_1$ about a subset $\bar{S}$ of our dataset $S$, i.e. $\cup_{(x,u,\dot{x}) \in \bar{S}} \, \mathcal{B}_{\bar{M}}((x,u),r_2)$ or a hyper-rectangle in the state-control space, depending on the data collection procedure.

We make the following assumptions. One, the true noise distribution is bounded. Two, in order to construct a valid control contraction metric (CCM, \cite{manchester}), we assume that $\X$ is a smooth manifold and the true dynamics are locally incrementally exponentially stabilizable, i.e. there exists a $\beta$, $\lambda$, and a feedback controller such that $\|x^*(t)-x(t)\| \leq \beta e^{-\lambda t} \|x^*(0) - x(0)\|$ in the trusted domain $D$. 
We do not assume the true system is control affine. Finally, we present the problem we wish to solve:

\noindent \textbf{Problem}: Given datasets $S$ and $V$, learn a dynamics model, design a tracking controller, and construct a trusted domain $D$ for planning. At planning time, given a start $x_I$, goal $x_G$, goal tolerance $\mu$, and safe set $\normalfont\Xsafe$, plan a nominal trajectory $x^*: [0, T] \rightarrow \X$, $u^*: [0, T] \rightarrow \U$ under the learned dynamics $\hat h$ such that $x^*(0) = x_I$, $\dot x^* = \hat h(x^*,u^*)$, $\|x^*(T) - x_G\| \leq \mu$, and $x^*(t)$, $u^*(t)$ remains in $D \cap \normalfont\Xsafe$ for all $ t \in [0, T]$. Also, guarantee that in tracking $(x^*(t), u^*(t))$ under the true dynamics $f$, the system remains in $D \cap \normalfont\Xsafe$ and reaches $\mathcal{B}(x_G, \hat\mu+\mu)$ where $\hat\mu$ is the tracking error bound.


\section{Method}

Our approach rests on learning three models: the approximate control affine dynamics $\hat{h}(x,u)$, the estimated disturbance bound $w(x,u)$%
, and the CCM $M(x)$. These models are used in a sampling based planner that computes tracking tubes around nominal plans. Certain properties of these models are verified using EVT which guarantees the tracking tube size and consequently safety in execution.

\subsection{Use of Extreme Value Theory}

The safety guarantees of this work rests on Extreme Value Theory, specifically the Fisher-Tippet-Gnedenko Theorem (FTG) \cite{de2007extreme} which guarantees that sample batch maximums of a random variable $Z$ converge to the generalized extreme value distribution (GEV).
We collect $N_s$ batches of size $N_b$ of i.i.d. samples, $z_i^j \in Z$, and fit the sample batch maximums, $\{\max_{1 \leq i \leq N_b} z_i^j\}_{j=1}^{N_s}$, to GEV
with shape, location, and scale parameters. If the fitted shape parameter denoted $\xi$ is less than zero, than the maximum of $Z$ is finite and can be overestimated with a given confidence $\psi$ from the location parameter. In practice, we also perform a Kolmogorov-Smirnov goodness-of-fit test \cite{degroot2013probability} to assure fit quality with $p=0.05$. See \cite{lipschitz_ral} for more details.

\subsection{Robust Model Characterization}

We train $\hat{h}$ using mean squared error loss and $w$ with a hinge loss to encourage the following relationship.
\begin{equation}\label{eq:dist_est}
    |h(x,u) - \hat{h}(x,u)| \leq w(x,u)
\end{equation}

$\leq$ in the equation above is taken to mean element-wise. Merely training does not guarantee the above relationship is satisfied, so we employ EVT to estimate the extreme value of the following distribution $Z$.
\begin{equation}\label{eq:violation_distribution}
    Z \sim \max_k (|h(x,u) - \hat{h}(x,u)| - w(x,u))_k
\end{equation}


Note that if $Z \leq 0$ with probability 1, then $w$ is never an underestimate of the maximum deviation from the learned model for any $(x,u)$ pair. By fitting sample batch maximums from our domain $D$ to GEV, we estimate the maximum of $Z$ up to the confidence $\psi$ which we call $\phi$. Then the following relationship is statistically guaranteed with confidence $\psi$:
\begin{equation}\label{eq:verified_dist}
    |h(x,u) - \hat{h}(x,u)| \leq w(x,u) + \phi
\end{equation}

This statement is much stronger than \eqref{eq:dist_est}, since it is verified for all $(x, u)$ in the domain $D$. We craft a robust planning and control system for the verified dynamics:
\begin{equation}\label{eq:verified_dyn}
    \hat{h}(x,u) = f(x) + g(x) u + \delta, \quad |\delta| \leq w(x,u) + \phi
\end{equation}


\subsection{Controller Formulation}


Our controller formulation is based on Control Contraction Metrics \cite{DBLP:journals/corr/abs-1912-13138}. The core approach requires a CCM $M(x)$, which is a Riemannian metric that satisfies some additional conditions (to be described shortly), and a corresponding tracking feedback controller $k(x, x^*, u^*) : \X \times \X \times \U \rightarrow \U$. We define the executed state trajectory as $x(t): [0, T] \rightarrow \X$. $k(x, x^*, u^*)$ returns a modified control input that stabilizes the current state $x(t)$ to the nominal state $x^*(t)$. The existence of such a CCM and a controller ensures that the distance between the nominal and executed trajectory decays exponentially, at some contraction rate $\lambda$, i.e., $d(x^*(t), x(t)) \leq C e^{-\lambda (t - t_0)} d(x^*(t_0), x(t_0)), t \geq t_0$, for some constant $C > 0$, when model error and noise are excluded. There are several sufficient conditions for ensuring that $M(x)$ is a CCM (see \cite{ChouOB21}); here we select the following: 
\begin{equation}\label{eq:contraction_conditions}\small
    \hspace{-10pt}\begin{aligned}
    g_\perp(x)^\top \Big( -\partial_f W(x) + 2\overline{\tfrac{\partial f(x)}{\partial x} W(x)} + 2 \lambda& W(x) \Big) g_\perp(x) \\
    &\doteq C_1(x) \preceq 0 \\
    g_\perp(x)^\top \left( \partial_{g^j} W(x) + 2 \overline{\tfrac{\partial g^j(x)}{\partial x} W(x)} \right) g_\perp(x) &= 0,\ j=1\ldots\dimu
    \end{aligned}\hspace{-10pt}
\end{equation}

\noindent where $W(x) = M^{-1}(x)$, $g_\perp(x)$ is a matrix such that $g_\perp^\top(x) g(x) = 0$ for all $x$, $g^j(x)$ is the $j$th column of $g(x)$, $\partial_f W(x)$ is the Lie derivative of $W$ with respect to $f$, and similarly for $\partial_{g^j} W(x)$. Intuitively, the first constraint is a contraction condition that is simplified by the second ``orthogonality" condition; together, they imply the system can be made to contract at rate $\lambda$ along arbitrary dynamically-feasible trajectories. In our results, the second condition is satisfied by thoughtful construction of $M(x)$ and $g(x)$: either $M(x)$ and $g(x)$ are constant, or we require $M(x)$ to be a function of only the first $\nx - \dimu$ states and select $g(x)$ to be of the form $[\bm{0}_{\dimu \times \nx - \dimu} \mathbf{I}_{\dimu \times \dimu}]^\top$ (see \cite{ChouOB21} for details). This construction is always possible without loss of generality by considering an augmented version of the original system where $u$ is treated as part of the state and $\dot{u}$ is treated as a new virtual input.

It is known \cite{manchester} that if $M(x)$ satisfies \eqref{eq:contraction_conditions} over $D_x$, then there always exists a controller $k(x, x^*, u^*)$ which makes trajectories contract at rate $\lambda$ everywhere in $D_x$. To implement $k$, we use an optimization-based controller from \cite[Eqn. (40)]{sumeet_icra}, which returns a minimum-norm modification $\ufb$ to the nominal control $u^*$, such that the executed control $u(t) = \ufb(t) + u^*(t)$ contracts the system at rate $\lambda$. We can also upper bound the magnitude of this controller based on the distance between an executed and nominal trajectory: $\|\ufb(t)\| \leq \|x(t) - x^*(t)\| \delta_u$; see (10) of \cite{ChouOB21} for the exact definition of $\delta_u$. Since \eqref{eq:contraction_conditions} needs to be true everywhere in $D_x$, but we can only evaluate \eqref{eq:contraction_conditions} point-wise, we verify satisfaction with EVT. We estimate the maximum of $\bar{\lambda}_D(C_1(x))$ over $D_x$, denoted $\bar{C}_1$, and verify that it is less than $0$ up to the given confidence $\psi$. We also apply EVT to estimate $\delta_u$ as in \cite{ChouOB21}.

\subsection{Robust Control under Disturbance}

In general, we find a valid CCM on the fixed learned dynamics by parameterizing $W(x)$ and $\lambda(x)$ as neural networks. To satisfy \eqref{eq:contraction_conditions} we only need the minimum value of $\lambda(x)$ but we found in practice that parameterizing $\lambda$ based on the state improved training performance. Our training loss is equal to $L(B) = \sum_{i=1}^3 L_i(B)$ where $B = \{x_1, \ldots, x_{N_b}\}$ is a training batch and $L_1$, $L_2$, and $L_3$ are defined as
\begin{equation}
    L_1(B) = - \alpha_1 \frac{1}{N_b} \sum_{i=1}^{N_b} \exp \left[ \max(\bar{\lambda}(C_1(x_i)), \tau) \right],
\end{equation}\vspace{-8pt}
\begin{equation}
    L_2(B) = \alpha_2 \frac{1}{N_b} \sum_{i=1}^{N_b} \sqrt{\frac{\bar{\lambda}(M(x_i))}{\underline{\lambda}(M(x_i))}},
\end{equation}
\begin{equation}
    L_3(B) = -\alpha_3 \underset{1 \leq i \leq N_b}{\min} \lambda(x_i),
\end{equation}

\noindent where $\max(\cdot,\cdot)$ returns the maximum of the two inputs. $\tau$ is a threshold selected to be $-0.01$ in our results. $\alpha_1$, $\alpha_2$, and $\alpha_3$ are each tunable weights. $L_1$ promotes the satisfaction of \eqref{eq:contraction_conditions} on the training data. $L_2$ promotes the CCM to be well-conditioned, which decreases the tube size for collision-checking. $L_3$ promotes large contraction rates.
 
While the obtained CCM provides a contracting controller for the \textit{learned dynamics}, the closed-loop behavior on the \textit{true system} will differ due to the model mismatch and stochastic disturbances. Fortunately, under bounded perturbations, the tracking error $\trackerr(t) \doteq d(x(t),x^*(t))$ satisfies the following relation (see \cite{chou2022safe}). Here, $\Theta(x)$ is the Cholesky decomposition of $M(x)$, i.e., $M(x) = \Theta(x)^\top \Theta(x)$:
\begin{equation}\label{eq:trackerr}
    \dot{\trackerr}(t) \leq - \lambda \trackerr(t) + \|\Theta(x(t)) (w(x(t),u(t)) + \phi)\|.
\end{equation}

In Sec. \ref{sec:method_planning}, we will leverage this relation to compute disturbance-informed plans that can be more safely tracked.
 
\subsection{Planner Formulation}\label{sec:method_planning}
 
We now discuss our planner, which uses the learned dynamics $\hat{h}$, verified error bound $w(x,u) + \phi$, CCM $M(x)$, and minimum contraction rate $\lambda$ to provide runtime assurances. We use a modified kinodynamic RRT (see Alg. \ref{alg:rrt}) that maintains the accumulated tracking error along transitions. We name this algorithm \underline{S}afety \underline{A}nalysis \underline{f}or L\underline{e}arned Stochastic Underactuated \underline{D}ynamics (SAFED) RRT. In the algorithm block, \texttt{SampleState} samples a random state in $D$, \texttt{NearestNeighbor} returns the nearest neighbor in the tree according to the standard Euclidean norm, \texttt{SampleCandidateControl} samples a random control in $\U$, \texttt{IntegrateLearnedDyn} integrates the learned dynamics returning the state at each $t \in [0,dt]$.

\subsubsection{Propagating Tracking Error}\label{sec:track_bound}

For the following section, we drop dependence on $t$ for brevity. To obtain tubes, we cannot directly use \eqref{eq:trackerr}, as it requires knowledge of the executed trajectory, which is unknown at planning time. We instead seek a tracking error bound that is valid for all possible executed trajectories in the tracking tube $T \doteq \{(x, u) \mid d(x,x^*) \leq \trackerr \,\text{and}\, \|u - u^*\| \leq \|x - x^*\| \delta_u\}$:
\begin{equation}\label{eq:plantrackerr}
    \dot{\trackerr} \leq - \lambda \trackerr + \underset{(x,u) \in T}{\max}\|\Theta(x) (w(x,u) + \phi)\|
\end{equation}
\noindent Computing this maximum is difficult as $\Theta$ and $w$ are complex NNs. For smaller NNs, we can upper bound the maximum via \cite{JordanD20, NIPS2019_9319}, but we opt for a derivation based on the triangle inequality with the nominal point on the path.
\begin{equation}\label{eq:track_dist_bound}\small
\hspace{-10pt}\begin{aligned}
    \max_{(x,u)\in T}&\|\Theta(x) (w(x,u) + \phi)\| \\
    &\hspace{-30pt}\leq \|\Theta(x^*)(w(x^*\hspace{-2pt}, u^*)+ \phi)\| + \|\Theta(x)w(x,u) - \Theta(x^*) w(x^*\hspace{-2pt}, u^*)\| \\
    &\hspace{-17pt}\leq \|\Theta(x^*)(w(x^*, u^*) + \phi)\| + L_{\Theta w} (d(x,x^*) + \|u - u^*\|)  \\
    &\hspace{-17pt}\leq \|\Theta(x^*)(w(x^*, u^*) + \phi)\| + L_{\Theta w} (1 + \tfrac{\delta_u}{\underline{\lambda}_D(M(x))}) d(x,x^*) 
\end{aligned}\hspace{-4pt}
\end{equation}

\noindent Here, $L_{\Theta w}$ is a Lipschitz constant of $\|\Theta(x) w(x,u)\|$ with respect to the Riemannian distance in $x$ and Euclidean distance in $u$:
\begin{equation}\small
    L_{\Theta w} \geq \underset{(x,u) \neq (x',u') \in D}{\sup} \frac{\|\Theta(x)w(x,u) - \Theta(x')w(x',u')\|}{d(x,x') + \|u - u'\|}.
\end{equation}

For systems with swift changes between small and large noise bounds, $L_{\Theta w}$ may be large. However, in our results, this approach provided tight tracking bounds. $L_{\Theta w}$ and $\underline{\lambda}_D(M(x))$ are also estimated via EVT, similar to $\phi$, $\bar{C}_1$, and $\delta_u$. Plugging \eqref{eq:track_dist_bound} into the RHS of \eqref{eq:plantrackerr} and integrating over time yields an upper bound on the tracking error $\bar\eta(t)$ for any plan $x^*: [0, T] \rightarrow \X$ at any time $t \in [0, T]$, i.e., 
\begin{equation}\label{eq:trk_bnd}
    d(x(t), x^*(t)) \le \bar \eta(t) \doteq \bar\eta(0) + \textstyle\int_0^t \textrm{RHS}(\tau) \textrm{d}\tau.
\end{equation} Evaluating \eqref{eq:trk_bnd} requires integrating a scalar ODE, and can be done efficiently in planning (see line 7 of Alg. \ref{alg:rrt}).

\subsubsection{Collision Checking}

We upper bound the nonconstant Riemannian metric $M(x)$ with a constant matrix $\bar{M} \in \PDM$ as in \cite{sumeet_icra}. The collision check for a point $x^*(t)$ on the nominal plan reduces to checking $\mathcal{B}_{\bar{M}}(x^*(t), \eta(t) / \underline{\lambda}_D(M(x))) \in \X_{\text{safe}}$. In theory, this needs to be checked for for all $t$, but in practice we check at discrete samples. See line 8 of Alg. \ref{alg:rrt}.

\subsubsection{Staying in the Trusted Domain}

We ensure that the nominal plan remains in the trusted domain where the estimates of $\phi$, $\bar{C}_1$, $\delta_u$, $L_{\Theta w}$, and $\underline{\lambda}_D(M(x))$ are statistically guaranteed. This is equivalent to the ball in domain check $\mathcal{B}_{\bar{M}}(x^*(t), \eta(t) / \underline{\lambda}_D(M(x))) \in D_x$. See line 10 of Alg. \ref{alg:rrt}.

\subsubsection{Respecting Controller Bounds}

To guarantee that the nominal plan can be safely tracked, we must ensure that the feedback control can be faithfully executed, i.e. $u(t) \in D_u$. Since we have an upper bound on the feedback term, this corresponds to the ball in domain check $\mathcal{B}(u^*(t), \delta_u \eta(t) / \underline{\lambda}_D(M(x)) \in D_u$. This also ensures our statistical guarantees hold for all feedback control values. See line 9 of Alg. \ref{alg:rrt}.


\begin{theorem}
Given $\phi$, $\bar{C}_1$, $\delta_u$, $L_{\Theta w}$, and $\underline{\lambda}_D(M(x))$ have been appropriately overestimated (or underestimated for $\underline{\lambda}_D(M(x))$) and a plan $x^*(t), u^*(t)$ found by Alg. \ref{alg:rrt}, the following holds for the executed trajectory: $x(t) \in \Xsafe$, $x(t) \in D_x$, and $u(t) \in D_u$ $\forall t \in [0,T)$.
\end{theorem}

\begin{proof}
Since the tracking error is appropriately bounded following the argument in Section \ref{sec:track_bound}, and the planner verifies that the tracking tube stays with $\Xsafe$ and $D$%
, the resulting trajectory is guaranteed to be safe and feasible.
\end{proof}

As the estimated constants are valid with confidence $\psi^5$ (using independent samples for each estimation), the plan is also guaranteed safe and feasible with confidence $\psi^5$. 

\vspace{-8pt}
\begin{algorithm}\small
\KwIn{$x_I$, $x_G$, $\eta_I$, $\lambda$, $dt$, $\mu$}
\SetKwFunction{SampleState}{SampleState}
\SetKwFunction{SampleCandidateControl}{SampleCandidateControl}
\SetKwFunction{NearestNeighbor}{NearestNeighbor}
\SetKwFunction{IntegrateLearnedDyn}{IntegrateLearnedDyn}
\SetKwFunction{InDX}{InDX}
\SetKwFunction{InDU}{InDU}
\SetKwFunction{TrkErrBnd}{TrkErrBnd}
\SetKwFunction{Model}{Model}
\SetKwFunction{ConstructPath}{ConstructPath}
\SetKwFunction{InCollision}{InCollision}

$\T \leftarrow \{(x_I, \eta_I)\}$, $\mathcal{P} \leftarrow \{\emptyset\}$ \\

\While{\upshape True}{
    $x_d \leftarrow$ \SampleState()\\
	$x_{\textrm{n}}, \bar\trackerr_\textrm{n} \leftarrow$ \NearestNeighbor{$\T$, $x_d$} \\ 
	$u_\textrm{c} \leftarrow $ \SampleCandidateControl()\\
	$x_\textrm{c}^*(t) \leftarrow$ \IntegrateLearnedDyn($x_{\textrm{n}}$, $u_\textrm{c}$, $dt$)\\
	$\bar\trackerr(t) \leftarrow$ \TrkErrBnd($\bar\trackerr_\textrm{n}$, $x_\textrm{c}^*(t)$, $u_\textrm{c}$)\\ 
	\lIf{
	    $\neg$\InCollision($x_\textrm{c}^*(t)$, $u_\textrm{c}$, $\bar\trackerr(t)$) $\forall t \in [0,dt]$ $\wedge$  \\
	    $\quad$\InDU($x_\textrm{c}^*(t)$, $u_\textrm{c}$, $\bar\trackerr(t)$) $\forall t \in [0,dt]$ $\wedge$ \\
	    $\quad$\InDX($x_\textrm{c}^*(t)$, $u_\textrm{c}$, $\bar\trackerr(t)$) $\forall t \in [0,dt]$ $\wedge$ \\
	    $\quad\|x^*_c(dt) - x_d\| < \|x_n - x_d\|$}{
	        $\T \leftarrow \T \cup (x^*(dt), \bar\trackerr(dt))$, $\mathcal{P} \leftarrow \mathcal{P} \cup \{u_c\}$
	    }
	\lIf{$\exists t, x_c^*(t) \in \mathcal{B}(x_G,\mu)$}{return plan}
}
\caption{SAFED RRT}\label{alg:rrt}
\end{algorithm}\vspace{-16pt}

\section{Results}
We present results on three systems: 1) a 10D simulated quadrotor, 2) a physical CrazyFlie quadrotor, and 3) a physical Clearpath Jackal. Our full method satisfies tracking bounds in all runs. We present the results of our method and comparisons to a na\"ive planner that does not consider the trusted domain nor tracking tubes. We evaluate the baseline \cite{ChouOB21} by estimating a Lipschitz constant of the model error. In each case, the shape parameter $\xi$ of the GEV distribution is positive, corresponding to an infinite Lipschitz constant, and thus, infinite tracking tubes and infeasible plans.

\subsection{10D simulated quadrotor}

We select the underactuated 10D quadrotor model from \cite{sumeet_icra} to evaluate our approach in moderately-high dimensions. Here, we assume knowledge of the nominal dynamics and fit a CCM using it (here, $\lambda = 1.29$); the true dynamics are perturbed with additive noise, sampled uniformly within state-dependent bounds: $\pm [0.01, 0.01,\allowbreak 0.01, |0.05s_1|, |0.05c_2|, |0.05s_3|, |0.07s_7|, |0.07c_8|, |0.07s_9| ]^\top$ (where $s_i, c_i$ are $\sin(x_i), \cos(x_i)$), if within $[-5, 5]$. $w$ is represented as a neural network with 3 hidden layers, each with 256 neurons. Using EVT on a validation dataset of size 100000, we establish a constant bound of $\phi = 0.0075$, with confidence $\psi = 0.99$. We select $D_x = \bigcup_{i=1}^N B(x_i, r_x)$, where $r_x = 0.3$, and $D_u = \U$ from \cite{sumeet_icra}.


Using our method, we compute 20 plans in an obstacle-filled environment (cf. Fig. \ref{fig:quad10d}) and execute each five times. See Table \ref{table:stats_crazyflie} for numerical results and Figs. \ref{fig:quad10d} for visualizations. The na\"ive planner leads to substantially larger tracking errors due to the larger model error. Comparing against \cite{ChouOB21}, the resulting fitted GEV distribution has shape $\xi=0.6154>0$ which corresponds to an infinite Lipschitz constant.

\begin{table}\centering
\begin{tabular}{ c | c | c}
  & SAFED & Na\"ive planner \\\hline
 \cellcolor{lightgray!50!} \hspace{-7pt}Avg. trk. err.\hspace{-5pt} & \cellcolor{lightgray!50!} 0.021 $\pm$ 0.009 (0.052)& \cellcolor{lightgray!50!} 5.47e3 $\pm$ 2.88e4 (2.48e5) \\ 
 \hspace{-7pt}Goal error\hspace{-5pt} & 0.038 $\pm$ 0.025 (0.117)& 7.24e3 $\pm$ 3.52e4 (3.95e5) \\ 
\end{tabular}
\caption{10D quadrotor tracking errors $\Vert x^*(t) - x(t)\Vert_2$. Mean $\pm$ standard deviation (worst case).}
\label{table:stats_quad10d}
\end{table} 
\begin{figure}
    \centering
    \includegraphics[width=0.9\linewidth]{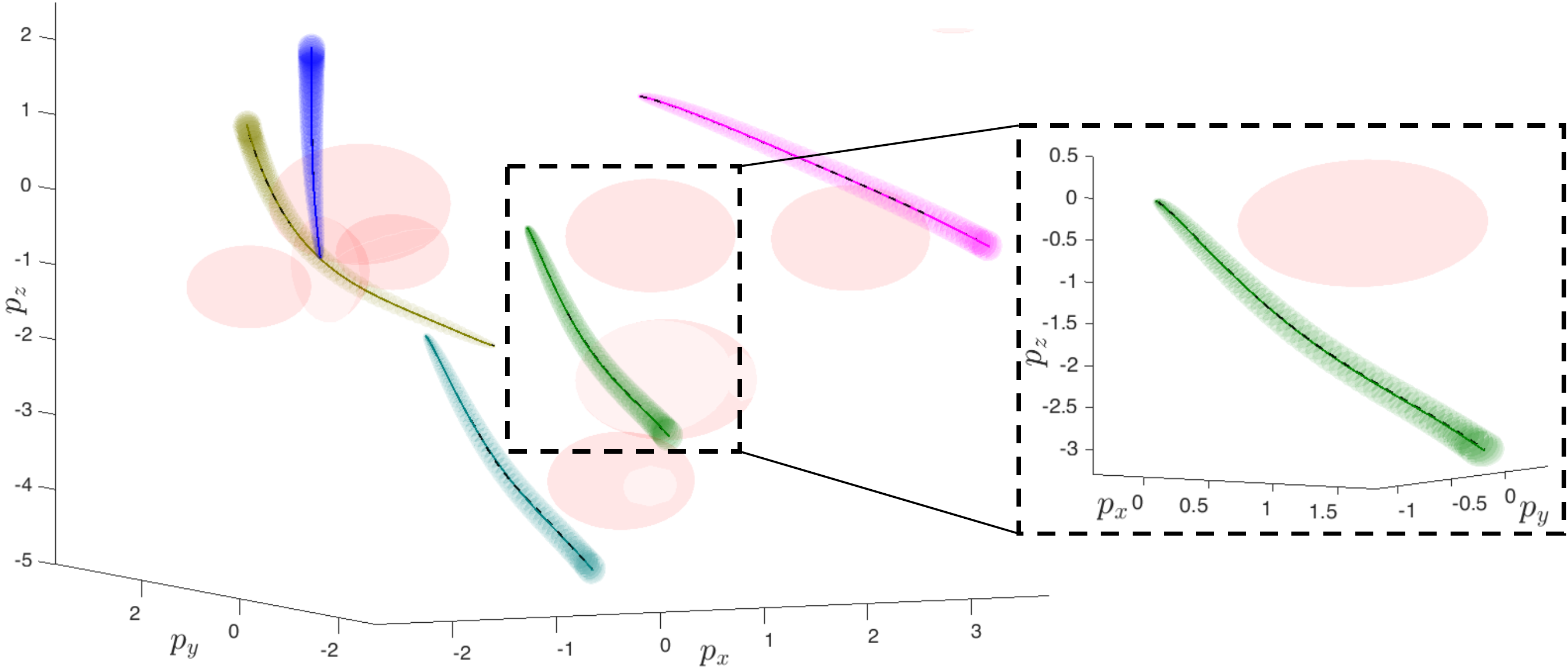}
    \caption{Five plans computed between random start/goals using SAFED; associated tracking tubes are color-coded to the plan. Five executions for each plan are overlaid in black (with similar results); obstacles are in red; zoomed-in view on right.\vspace{-18pt} }
    \label{fig:quad10d}
\end{figure}

\subsection{Physical CrazyFlie quadrotor}

We consider feedback motion planning for a learned model of a physical CrazyFlie quadrotor, moving in 3D. We select the state as the linear positions and velocities $x = [p_x, p_y, p_z, \dot p_x, \dot p_y, \dot p_z]$, and represent the learned dynamics with a linear model, $\dot x = Ax + Bu$, where $A$ and $B$ were fit to 5700 datapoints using linear least squares. 
Here, $w$ is represented as a neural network with five hidden layers, each with 1024 neurons. Using EVT on a validation dataset of size 10800, we establish a constant bound of $\phi = 0.0178$, with confidence $\rho = 0.99$; to obtain approximately i.i.d. samples in $D$ required by EVT, we subsample the trajectory rollouts (justified by \cite{sattar2020non}). We select $D_x = \bigcup_{i=1}^N B(x_i,r_x)$ and $D_u = \bigcup_{i=1}^N B(u_i,r_u)$, and $r_x = 1.5$ and $r_u = 0.1$; these values are selected as the $90$th percentile of the minimum distance of each validation datapoint to the training data. Given the learned linear dynamics, the CCM conditions simplify to a single semidefinite constraint (i.e., the condition \eqref{eq:contraction_conditions} holds uniformly for all $x$); hence, the CCM is valid everywhere for the learned dynamics; we set $\lambda = 2.0$. 

Using our method, we compute 20 plans in an obstacle-filled environment (cf. Fig. \ref{fig:crazyflie}) and execute each three times. See Table \ref{table:stats_crazyflie} for numerical results and Figs. \ref{fig:crazyflie} and \ref{fig:crazyflie_real} for visualizations. 
In executing these plans, we terminate automatically if the drone deviates from the planned trajectory by more than $0.2$ in any position dimension for platform safety. The na\"ive plans lead to premature terminations in each run (cf. Table \ref{table:stats_crazyflie}. Comparing against \cite{ChouOB21}, the resulting fitted GEV distribution has shape $\xi=0.3272>0$ which corresponds to an infinite Lipschitz constant.

\begin{table}\centering
\begin{tabular}{ c | c | c }
  & SAFED & Na\"ive planner \\\hline
 \cellcolor{lightgray!50!} \hspace{-7pt}Avg. trk. err.\hspace{-5pt} & \cellcolor{lightgray!50!} 0.066 $\pm$ 0.012 (0.121)& \cellcolor{lightgray!50!} 0.198 $\pm$ 0.044 (0.294) \\ 
 \hspace{-7pt}Goal error\hspace{-5pt} & 0.107 $\pm$ 0.042 (0.203)& -- \\ 
\end{tabular}
\caption{Crazyflie tracking errors $\Vert x^*(t) - x(t)\Vert_2$. Mean $\pm$ standard deviation (worst case).}
\label{table:stats_crazyflie}
\end{table} 

\begin{figure}
    \centering
    \begin{subfigure}[b]{0.49\linewidth}
        \includegraphics[width=\linewidth]{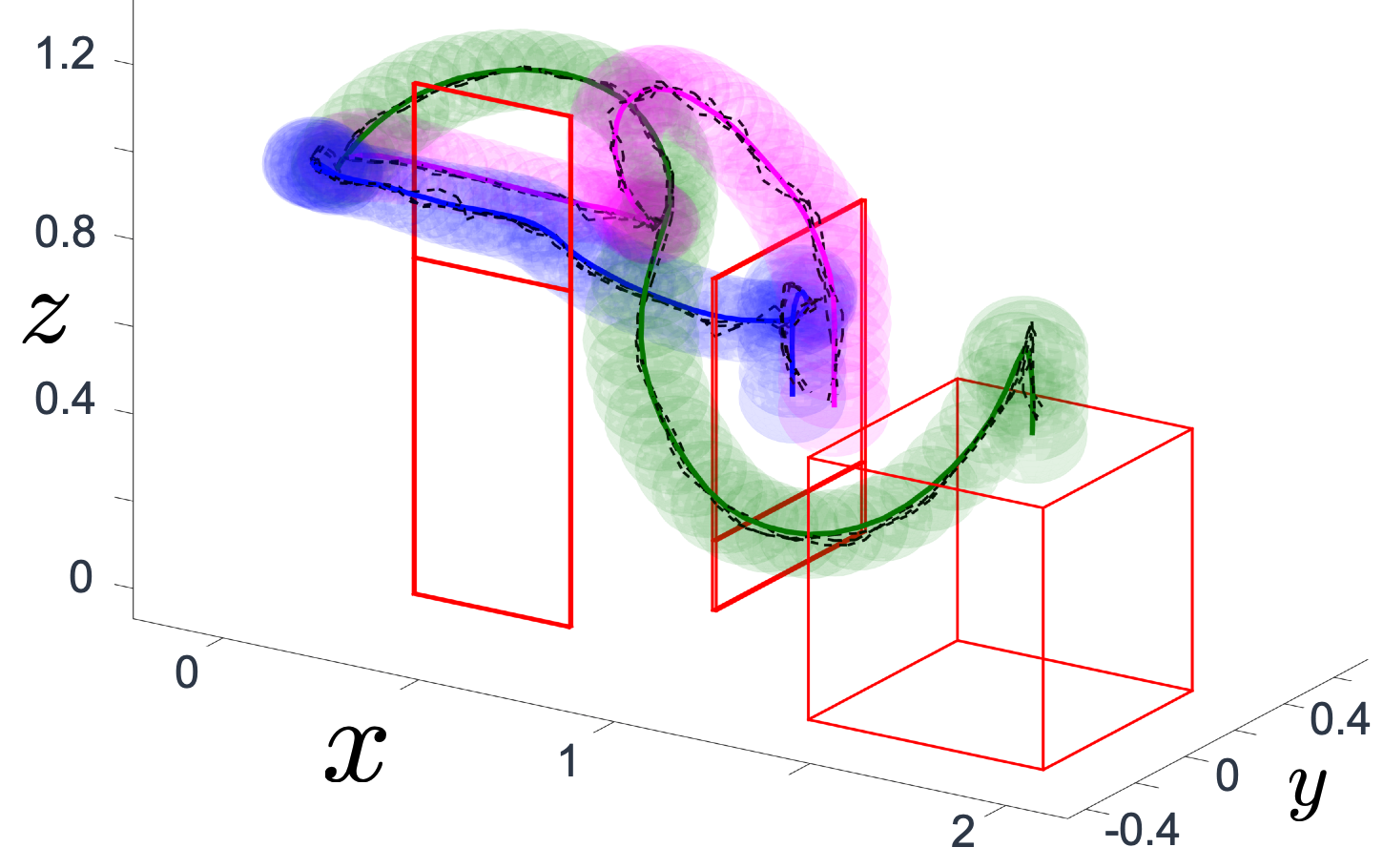}
        \caption{}\label{fig:crazyflie}
    \end{subfigure}
    \begin{subfigure}[b]{0.49\linewidth}
        \includegraphics[width=\linewidth]{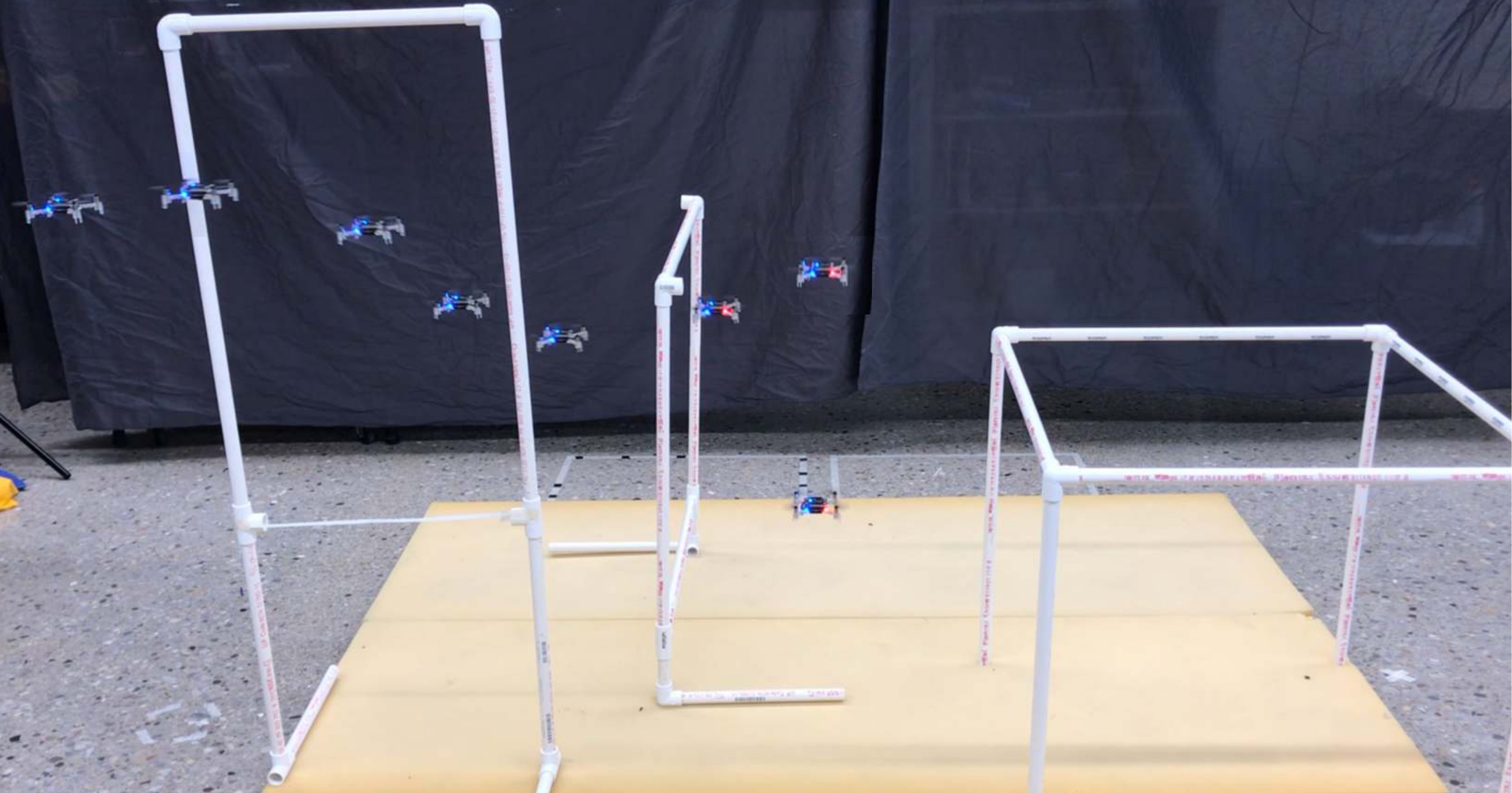}
        \caption{}\label{fig:crazyflie_real}
    \end{subfigure}
    \caption{(a) Three plans computed using SAFED; the corresponding tracking tubes are color-coded to the plan. The three executions for each plan are overlaid in black; the obstacles are in red. (b) One execution of the blue plan.}\vspace{-12pt}
\end{figure}


\subsection{Clearpath Jackal}

\begin{figure}
    \centering\vspace{-8pt}
    \includegraphics[width=\linewidth]{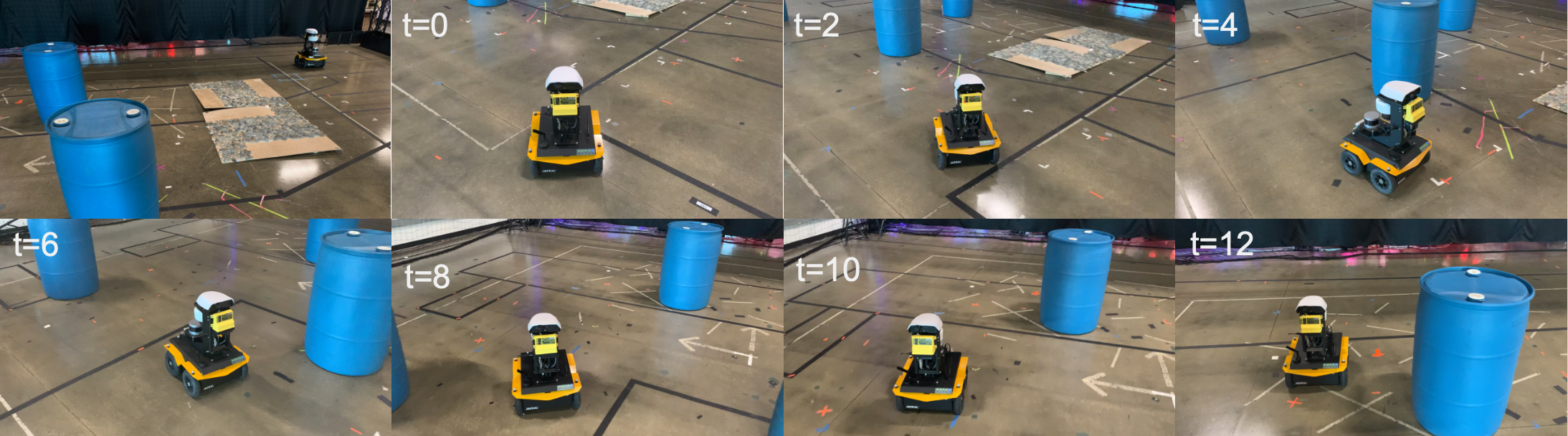}
    \caption{Upper left shows jackal test environment showing two barrels, rough terrain simulant, and Clearpath Jackal. Remaining figures show an example rollout.}
    \label{fig:jackal_test_env}
\end{figure}

We apply our method to a Clearpath Jackal. Our environment consisted primarily of cement but also included an approximately 1.25x2.5 meter area with artificial rough terrain, see Figure \ref{fig:jackal_test_env}. The Jackal takes linear $v$ and angular $w$ velocities as input. We model the dynamics as:
\begin{equation}
    \begin{bmatrix}
    x & y & \theta & v & w
    \end{bmatrix}^\top =
    \begin{bmatrix}
    f(x,u) \\
    \bm{0}_{2 \times 1}
    \end{bmatrix} +
    \begin{bmatrix}
    \bm{0}_{3 \times 2} \\
    \mathbf{I}_{2 \times 2}
    \end{bmatrix}
    \begin{bmatrix}
    a \\
    \alpha
    \end{bmatrix}
\end{equation}

where $x$ and $y$ correspond to the position  and $\theta$ the heading of the robot. $a$ and $\alpha$ are virtual acceleration inputs. $f$ is parameterized as a NN with three hidden ReLU layers of sizes 128, 256, and 128. Our metric $M(x)$ is parameterized as a NN, using only $x$, $y$, and $\theta$ as input in order to satisfy the orthogonality condition of \eqref{eq:contraction_conditions}. The network outputs the lower left Cholesky decomposition $L$ of the dual metric $W$. The diagonal elements of $L$ are computed via three hidden ReLU layers of sizes 256, 512, and 256 and a final SoftPlus output layer. Off-diagonal elements are computed with another network with three hidden ReLU layers of sizes 256, 512, and 256. $w$ is composed of one hidden ReLU layer of size 128. We collected random rollouts on the data resulting in about 22000 training points and 2200 validation points over the hyperrectangle $[-4.2,3.2]\times[-3.1,3.1]\times[-\pi,\pi]\times[0.3, 1.2]\times[-1.25,1.25]\times[-0.6, 0.6]\times[-0.8, 0.8] = D \subset \X \times \U$. Our dynamics and CCM are trained with data without the rough terrain simulant. We collected an additional 11000 training points and 1100 validation points with the rough terrain simulant to train the bound $w(x,u)$. With confidence $\psi = 0.99$, we found $\phi = 0.000204$, $\lambda = 0.3277$, $L_{\Theta w} = 0.046$, $\delta_u = 0.106$, $\underline{\lambda}_D(M(x))=0.1508$, and $\bar{C}_1 = -0.0037$.

Using our full method, we computed 8 plans in the environment with 3 additional barrel obstacles and executed each plan 3 times, see Figure \ref{fig:jackal_rollouts}. No plans were found over the rough terrain region, indicating that the model learned to avoid larger noise as it contributed to large and unsafe tracking tubes.  We compare to the na\"ive planner and a second ``unsafe" planner, computing 5 plans for each one and executing each plan 3 times. The unsafe planner only uses the nominal trajectory for collision checking rather than the full tracking tube. Tracking results are given in Table \ref{table:stats_jackal} for the full and na\"ive planners. The tracking results for the unsafe planner are similar to the full method, but 40\% of rollouts collided with the barrels compared to no rollouts of the full method.  Comparing against \cite{ChouOB21}, the resulting fitted GEV distribution has shape $\xi=0.1582>0$ which corresponds to an infinite Lipschitz constant of the model error.

\begin{figure}
    \centering
    \begin{subfigure}[b]{0.47\linewidth}
        \includegraphics[width=\linewidth]{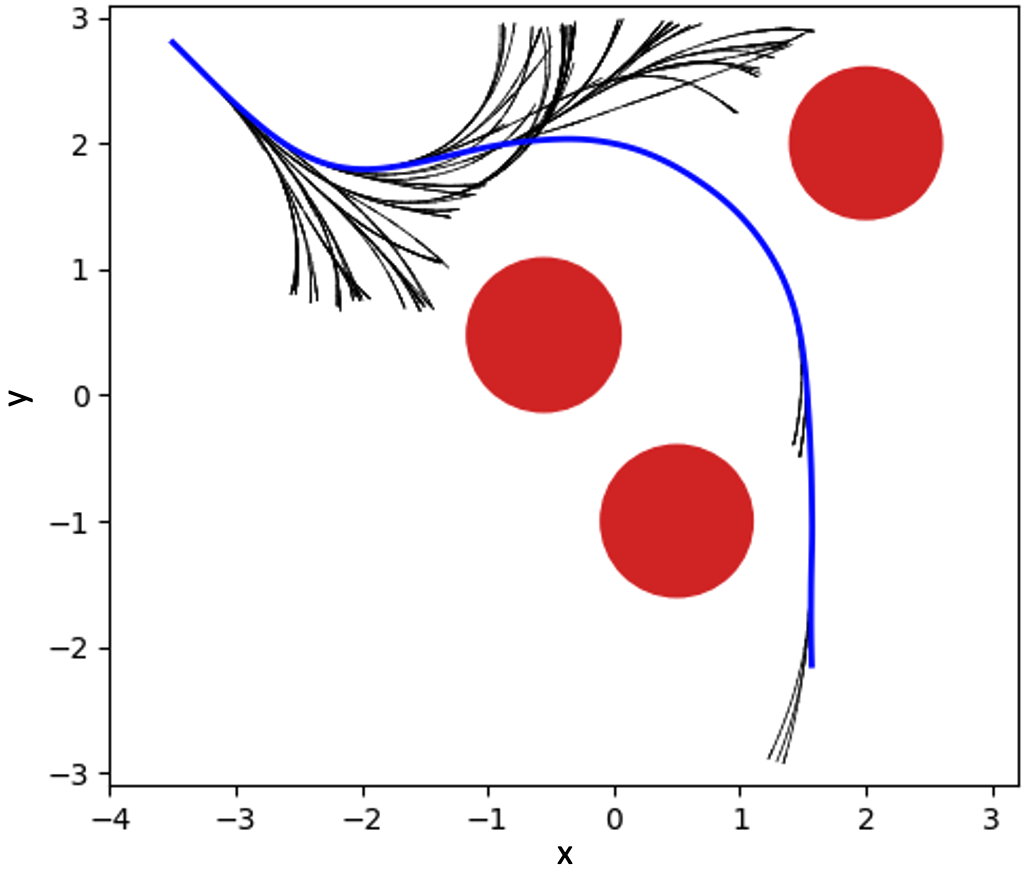}
    \end{subfigure}
    \begin{subfigure}[b]{0.51\linewidth}
        \includegraphics[width=\linewidth]{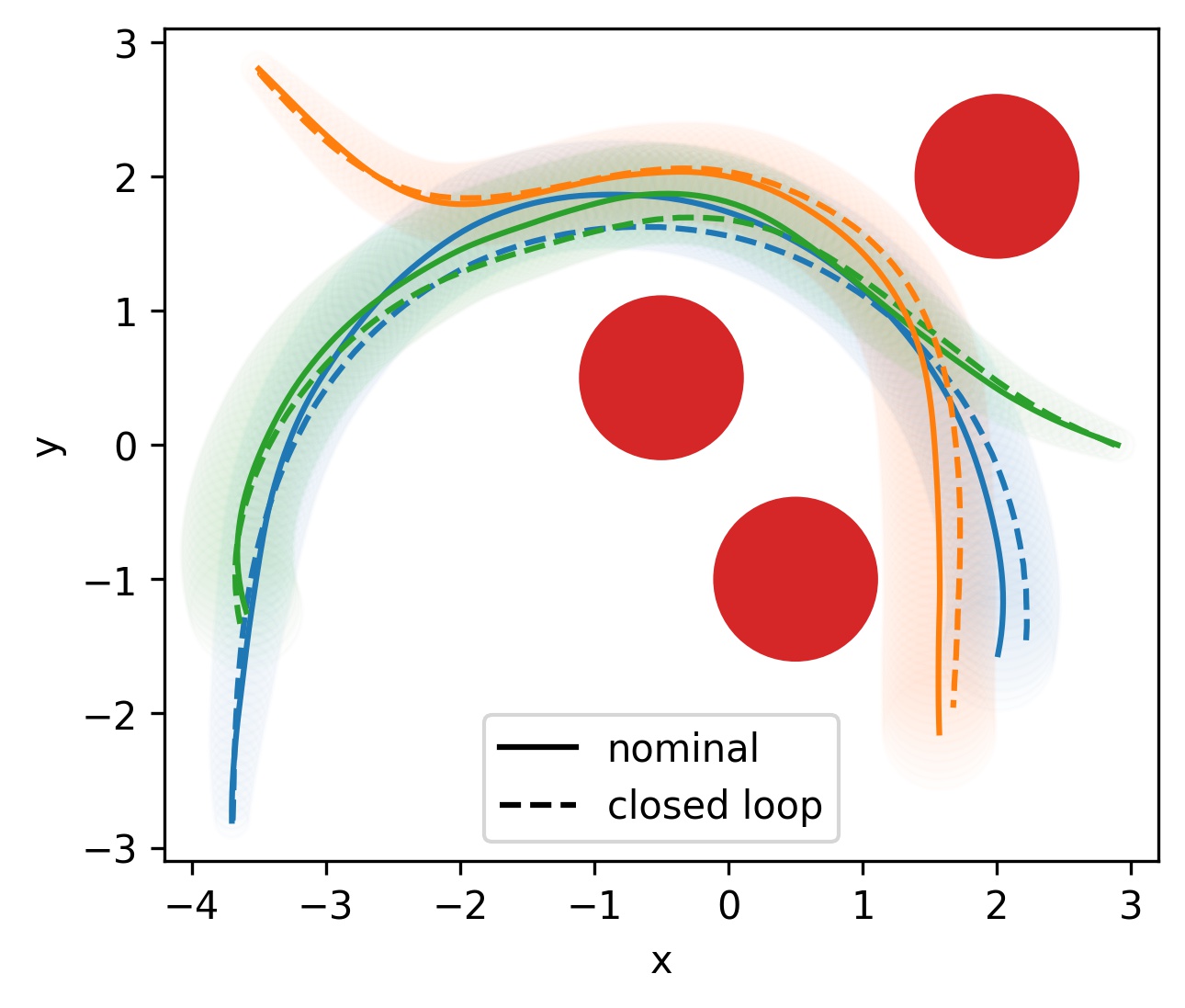}
    \end{subfigure}
    \caption{Left: Example tree and found plan. Notice the tree stops expanding in the area with rough terrain. Right: Three example plans and corresponding rollouts using our full method. Tube used for collision checks is shown with shaded region. Obstacles are in red. Not shown is rough terrain region approximately spanning $[-2.5,-1.25]$ in $x$ and $[-2.1, 0.3]$ in $y$.}
    \label{fig:jackal_rollouts}
    \vspace{-8pt}
\end{figure}

\begin{table}\centering
\begin{tabular}{ c | c | c }
  & SAFED & Na\"ive planner \\\hline
 \cellcolor{lightgray!50!} \hspace{-7pt}Avg. trk. err.\hspace{-5pt} & \cellcolor{lightgray!50!} 0.193 $\pm$ 0.0529 (0.256)& \cellcolor{lightgray!50!} 0.512 $\pm$ 0.450 (1.514) \\ 
 \hspace{-7pt}Goal error\hspace{-5pt} & 0.187 $\pm$ 0.0556 (0.262)& 1.24 $\pm$ 1.06 (3.60) \\ 
\end{tabular}
\caption{Jackal tracking errors $\Vert x^*(t) - x(t)\Vert_2$. Mean $\pm$ standard deviation (worst case).}
\label{table:stats_jackal}
\vspace{-16pt}
\end{table}

\vspace{-8pt}
\section{Conclusion}
\vspace{-3pt}

We present a method that is capable of planning safe, feasible paths for stochastic, underactuated, unknown dynamics. We verify the properties of a learned dynamics model, noise bound, and control contraction metric using Extreme Value Theory to bound tracking tubes around nominal plans. The planner ensures that the tubes do not collide, the tubes stay in the domain where our statistical guarantees hold, and feedback control does not exceed controller bounds. We experimentally verify on 10D simulated quadrotor, a CrazyFlie, and a Clearpath Jackal. While this method is capable of handling stochastic systems, it does so in a robust way by always considering worst case. Potentially the conservativeness can be reduced by learning the distribution of noise and utilizing probabilistic collision bounds in planning. 
This method could also be extended to a real-time receding horizon planner in order to plan in unknown environments.

\bibliographystyle{IEEEtran}
\bibliography{main}

\end{document}